\documentclass{article} 
\usepackage{times}
\usepackage{graphicx, xr} 
\usepackage{verbatim}
\usepackage{hyperref}

\usepackage{amsfonts}       

\usepackage{amsmath}

\usepackage{color}
\usepackage{graphicx} 
\usepackage{hyperref}
\usepackage{amsmath}
\usepackage{booktabs} 
\usepackage{array} 
\usepackage{paralist} 
\usepackage{verbatim} 
\usepackage{subfig} 
\usepackage[ruled,vlined]{algorithm2e}
\newcommand{\lh}{L_2(\h,\mu)}
\newcommand{\be}{\begin{equation}}
\newcommand{\bea}{\begin{eqnarray}}
\newcommand{\eea}{\end{eqnarray}}
\newcommand{\bean}{\begin{eqnarray*}}
\newcommand{\eean}{\end{eqnarray*}}

\newcommand{\Var}[1]{V\left[ {#1}\right]}
\newcommand{\ee}{\end{equation}}
\newcommand{\p}{\mu}
\newcommand{\x}{\mathcal{X}}
\newcommand{\y}{\mathcal{Y}}
\newcommand{\h}{\Omega}
\newcommand{\fwx}[1]{\psi_{{#1}}}
\renewcommand{\fwx}[1]{\Phi({#1})}

\newcommand{\roi}[1]{\textcolor{red}{#1}}

\usepackage{amssymb,amsthm}
\newtheorem{theorem}{Theorem}
\newtheorem{corollary}{Corollary}
\newtheorem{lemma}{Lemma}
\newtheorem{definition}{Definition}
\newcommand{\algref}[1]{Alg. \ref{#1}}
\renewcommand{\eqref}[1]{Eq.~\ref{#1}}
\newcommand{\lemref}[1]{Lemma \ref{#1}}
\newcommand{\thmref}[1]{Theorem \ref{#1}}
\renewcommand{\algref}[1]{Algorithm \ref{#1}}
\newcommand{\secref}[1]{Section \ref{#1}}
\newcommand{\apref}[1]{Appendix \ref{#1}}

\newcommand{\corref}[1]{Corollary \ref{#1}}

\newcommand{\reals}{\mathbb{R}}
\newcommand{\E}{\mathop\mathbb{E}}
\newcommand{\<}{\langle}
\renewcommand{\>}{\rangle}
\newcommand{\scalar}[2]{\left\<{#1},{#2}\right\>}
\newcommand{\norm}[1]{\left\|{#1} \right\|}
\newcommand{\ignore}[1]{}
\newcommand{\es}[2]{{\prec}{#1},{#2}{\succ}}
\renewcommand{\es}[2]{\psi(#1;\bar{\ww}) \psi(#2;\bar{\ww})}
\newcommand{\ww}{\mathbf{w}}
\newcommand{\vv}{f}
\renewcommand{\aa}{a}
\newcommand{\xx}{\mathbf{x}}
\newcommand{\xxp}{\fwx{\xx}}
\newcommand{\xxtp}{\fwx{\xx_t}}
\newcommand{\xxip}{\fwx{\xx_i}}

\newcommand{\action}[2]{\psi(#1;#2)}
\newcommand{\cclass}{\mathcal{H}^B_\mu}
\newcommand{\alphat}{\alpha^{(t)}}

\newcommand{\prob}[1]{{\mathbb P}\left[{#1}\right]}
\newcommand{\expect}[1]{{\mathbb E}\left[{#1}\right]}

\usepackage[accepted]{icml2017} 

\icmltitlerunning{Learning Infinite Layer Networks Without the Kernel Trick}

\begin{document}
\twocolumn[
\icmltitle{Learning Infinite Layer Networks Without the Kernel Trick}

\begin{icmlauthorlist}
\icmlauthor{Roi Livni}{pr}
\icmlauthor{Daniel Carmon}{tau}
\icmlauthor{Amir Globerson}{tau}
\end{icmlauthorlist}

\icmlaffiliation{pr}{University of Princeton, Princeton, New Jersey, USA}
\icmlaffiliation{tau}{Tel-Aviv University, Tel-Aviv, Israel}
\icmlcorrespondingauthor{Roi Livni}{rlivni@cs.princeton.edu}
\icmlcorrespondingauthor{Daniel Carmon}{carmonda@mail.tau.ac.il}
\icmlcorrespondingauthor{Amir Globerson}{gamir@mail.tau.ac.il}
\vskip 0.3in
]
\printAffiliationsAndNotice{}

\begin{abstract}

Infinite Layer Networks (ILN) have been proposed as an architecture that mimics neural networks while enjoying some of the advantages of kernel methods. ILN  are networks that integrate over infinitely many nodes within a single hidden layer. It has been demonstrated by several authors that the problem of learning ILN can be reduced to the kernel trick, implying that whenever a certain integral can be computed analytically they are efficiently learnable. 
In this work we give an online algorithm for ILN, which avoids the \emph{kernel trick assumption}. More generally and of independent interest, we show that kernel methods in general can be exploited even when the kernel cannot be efficiently computed but can only be estimated via sampling. We provide a regret analysis for our algorithm, showing that it matches the sample complexity of methods which have access to kernel values.  Thus, our method is the first to demonstrate that the kernel trick is not necessary, as such, and random features suffice to obtain comparable performance.


\end{abstract}

\section{Introduction}
With the increasing success of highly non-convex and complex learning architectures such as neural networks, there is an increasing effort to further understand and explain the limits of training such hierarchical structures.

Recently there have been attempts to draw mathematical insight from kernel methods in order to  better understand deep learning, as well as come up with new computationally learnable architectures. One such line of work consists of learning
 classifiers that are linear functions of a very large or infinite collection of non-linear functions \cite{bach2014breaking,daniely2016toward,cho2009kernel,heinemann2016improper,williams1997computing}. Such models can be interpreted as a neural network with infinitely many nodes in a hidden layer, and we thus refer to them as ``Infinite Layer Networks'' (ILN). They are of course also related to kernel based classifiers, as will be discussed later. 

A target function in an  ILN class will be of the form:
\begin{equation}\label{eq:infNet}
\xx\to \int \psi(\xx; \ww) f(\ww) d\mu(\ww),
\end{equation}
Here $\psi$ is some function of the input $\xx$ and parameters $\ww$, and $d\mu(\ww)$ is a prior over the parameter space. For example, $\psi(\xx; \ww)$ can be a single sigmoidal neuron or a complete convolutional network. The integral can be thought of as an infinite sum over all such possible networks, and $f(\ww)$ can be thought of as an infinite output weight vector to be trained.

A Standard $1$--hidden layer network with a finite set of units can be obtained from the above formalism as follows. First, choose $ \psi(\xx; \ww)  = \sigma(\xx\cdot\ww)$ where $\sigma$ is an activation function (e.g., sigmoid or relu). Next, set $d\mu(\ww)$ to be a discrete measure over a finite set $\ww_1,\ldots, \ww_d$.\footnote{In $\delta$ function notation $d\mu(\ww) = \frac{1}{d}\sum_{i=1}^d \delta(\ww-\ww_i)d\ww$} In this case, the integral results in a network with $d$ hidden units, and the function $f$ is
the linear weights of the output layer. Namely:
\[\xx \to \frac{1}{d} \sum_{i=1}^d f(\ww_i) \cdot \sigma(\xx\cdot \ww_i).\]

The main challenge when training $1$--hidden layer networks is of course to \emph{find} the $\ww_1,\ldots,\ww_d$ on which we wish to support our distribution.  It is known \cite{livni2014computational}, that due to hardness of learning intersection of halfspaces \cite{klivans2006cryptographic,daniely2014average}, $1$--hidden layer neural networks are computationally hard for a wide class of activation functions.  Therefore, as the last example illustrates, the choice of $\mu$ is indeed crucial for performance.

For a fixed prior $\mu$, the class of ILN functions is highly expressive, since $f$ can be chosen to approximate any 1-hidden layer architecture to arbitrary precision (by setting $f$ to delta functions around the weights of the network, as we did above for $\mu$). However, this expressiveness comes at a cost. As argued in \citet{heinemann2016improper}, ILN will generalize well when there is a large probability mass of $\ww$ parameters that attain a small loss. 

The key observation that makes certain ILN tractable to learn is that \eqref{eq:infNet} is a linear 
functional in $f$. In that sense it is a linear classifier and enjoys the rich theory and algorithmic
toolbox for such classifiers. In particular, one can use the fact that linear classifiers can be learned via the kernel trick in a batch \cite{cortes1995support} as well as online settings \cite{kivinen2004online}. In other words, we can reduce learning ILN to the problem of computing the kernel function between two examples. Specifically the problem reduces to computing integrals of the following form:
\bea\label{eq:kernelTrick}
k(\xx_1,\xx_2) &=& \int \psi(\xx_1; \ww) \cdot \psi(\xx_2; \ww) d\mu(\ww) \\
&=& \E_{\bar{\ww}\sim \mu}\left[\psi(\xx_1; \bar{\ww})\cdot \psi(\xx_2 ; \bar{\ww})\right].
\eea
In this work we extend this result to the case where no closed form kernel is available, and thus the kernel trick is not directly applicable. We thus turn our attention to the setting
where features (i.e., $\ww$ vectors) can be randomly sampled.  In this setting, our main result shows that for the squared loss, we can efficiently learn the above class. Moreover, we can surprisingly do this with a computational cost comparable to that of methods that have access to the closed form kernel $k(\xx_1,\xx_2)$. 

The observation we begin with is that sampling random features (i.e., $\ww$ above), leads to an unbiased estimate of the kernel in \eqref{eq:kernelTrick}. Thus, if for example, we ignore complexity issues and can sample infinitely many $\ww$'s, it is not surprising that we can avoid the need for exact computation of the kernel. 
However, our results provide a much stronger and practical result. Given $T$ training samples, the lower bound on achievable accuracy is $O(1/\sqrt{T})$ \citep[see][]{shamir2014sample}. We show that
we can in fact achieve this rate, using $\tilde{O}(T^2)$ calls\footnote{We use $\tilde{O}$ notation to suppress logarithmic factors} to the random feature generator. For comparison, note that $O(T^2)$ is the size of the kernel matrix, and is thus likely to be the cost of any algorithm that uses an explicit kernel matrix, where one is available. As we discuss later, our approach improves on previous random features based learning \cite{dai2014scalable,rahimi2009weighted} in terms of sample/computational complexity, and expressiveness.

\section{Problem Setup}
We consider algorithms that learn a mapping from input instances $\xx\in \x$ to labels $y\in \y$. We focus on the regression case where $\y$ is the interval $[-1,1]$. Our starting point is a class of feature functions $\action{\ww}{\xx}:  \h\times \x \to \mathbb{R}$, parametrized by vectors $\ww \in \h$. The functions $\action{\ww}{\xx}$ may contain highly complex non linearities, such as multi-layer networks consisting of convolution and pooling layers. Our only assumption on $\action{\ww}{\xx}$ is that for all $\ww\in \h$ and $\xx\in \x$ it holds that $|\action{\ww}{\xx}|<1$.

Given a distribution $\mu$ on $\h$, we denote by $\lh$ the class of square integrable functions over $\h$.
\[\lh = \left\{f: \int f^2(\ww)d\mu(\ww)<\infty\right\}.\]
We will use functions $f\in\lh$ as mixture weights over the class $\h$, where each $f$ naturally defines a new regression function from $\xx$ to $\reals$ as follows:
 \begin{equation}\label{eq:identification} \xx\to   \int \action{\ww}{\xx}f(\ww)d\mu(\ww).\end{equation}

Our key algorithmic assumption is that the learner can efficiently sample random $\ww$  according to the distribution $\p$. Denote the time to generate
one such sample by $\rho$.

In what follows it will be simpler to express the integrals as scalar products. Define the following scalar product on functions $f\in\lh$. 
\be
\scalar{f}{g} = \int f(\ww)g(\ww)d\mu(\ww)
\ee
We denote the corresponding $\ell_2$ norm by $\norm{f} = \sqrt{\scalar{f}{f}}$. Also, given features $\xx$ denote by $\fwx{\xx}$ the function in $\lh$ given by $\fwx{\xx}[\ww]=\action{\ww}{\xx}$. The regression functions we are considering are then of the form $\xx \to \scalar{f}{\fwx{\xx}}$.

A subclass of norm bounded elements in $\lh$ induces a natural subclass of regression functions. Namely, we consider the following class:
\[ \cclass= \left\{\xx\to \scalar{f}{\fwx{\xx}}: \|f\|<B\right\}.\]
Our ultimate goal is to output a predictor $f\in \lh$ that is competitive, in terms of prediction, with the best target function in the class $\cclass$.

We will consider an online setting, and use it to derive generalization bounds via standard online to batch conversion. In our setting, at each round a learner chooses a target function $f_t\in \lh$ and an adversary then reveals a sample $\xx_t$ and label $y_t$. The learner then incurs a loss of 
\be
\ell_t(f_t) = \frac{1}{2}\left(\scalar{f_t}{\fwx{\xx_t}} - y_t\right)^2.
\label{eq:loss}
\ee
The use of squared loss might seem restrictive if one is interested in classification. However, $L_2$ loss is common by now in classification with support vector machines and kernel methods since \citep{suykens1999least, suykens2002weighted}. More recently \citet{zhang2016understanding} showed that when using a large number of features regression achieves performance comparable to the corresponding linear classifiers (see Section 5 therein).

The objective of the learner is to minimize her $T$ round regret w.r.t norm bounded elements in $\lh$. Namely:
\be
\sum_{t=1}^T \ell_t(f_t) - \min_{f^*\in \cclass} \sum_{t=1}^T \ell_t(f^*).
\label{eq:regret}
\ee

In the statistical setting we assume that the sequence $S=\{(\xx_i,y_i)\}_{i=1}^T$ is generated IID according to some unknown distribution $\mathbb{P}$. We then define the expected loss of a predictor as
\begin{equation}\label{eq:statistical} L(f)= \E_{(\xx,y)\sim \mathbb{P}} \left[\frac{1}{2}\left(\scalar{f}{\fwx{\xx}} - y\right)^2\right].\end{equation}

\section{Main Results}
\label{sec:results}
\thmref{thm:mainmain} states our result for the online model. The corresponding result for the statistical setting is given in \corref{cor:mainmain}. We will elaborate on the structure of the Algorithm later, but first provide the main result.
{
\begin{algorithm}[h]
\KwData{$ T,~ B>1, \eta , m$}
 \KwResult{Weights $\alpha^{(1)},\ldots,\alpha^{(T+1)}\in\reals^{T}$. 
   Functions $f_t \in \lh$ defined as $f_t = \sum_{i=1}^t \alpha_i^{(t)} \fwx{\xx_i}$\;}
 Initialize $\alpha^{(1)}=\bar{0} \in \mathbb{R}^T$\;
 \For{$t=1,\ldots, T$}{
  Observe  $\xx_t,y_t$\;
  Set $E_t=\mathrm{EST\_SCALAR\_PROD}(\alpha^{(t)},\xx_{1:{t-1}},\xx_t,m)$;\\
   \uIf{$|E_t|<16 B $}{
	$\alpha^{(t+1)}=\alpha^{(t)}$\;
	 $\alpha^{(t+1)}_t = -\eta (y_t - E_t)$; 
  }  \Else{$\alpha^{(t+1)}=\frac{1}{4}\alpha^{(t)}$;}
   }

 \caption{The SHRINKING\_GRADIENT algorithm.} \label{alg:main}
\end{algorithm}

\begin{algorithm}[h]
\caption{EST\_SCALAR\_PROD}

\KwData{$\alpha$,~$\xx_{1:{t-1}}$,~$\xx$,~$m$}
\KwResult{Estimated scalar product $E$ }
\If{$\alpha = \bar{0}$}{
 Set $E=0$
}
\Else{
\For{k=1.\ldots,m}{
 Sample $i$ from the distribution $q(i) = \frac{|\alpha_i|}{\sum |\alpha_i|}$ \;
 Sample parameter $\bar{\ww}$ from $\mu$. Set $E^{(k)} = \mathrm{sgn}(\alpha_i)\es{\xx_i}{\xx}$;\
 }
 Set $E= \frac{\|\alpha\|_1}{m} \sum_{k=1}^m E^{(k)}$\label{alg:spsub}
}
\end{algorithm}
}

\begin{theorem}\label{thm:mainmain} 
Run \algref{alg:main} with parameters $T$, $B\ge 1$, $\eta= \frac{B}{\sqrt{T}}$ and $m = O\left(B^4 T\log \left(BT\right)\right)$.
Then:

\begin{enumerate}
\item For every sequence of squared losses $\ell_1,\ldots, \ell_T$ observed by the algorithm we have for $f_1,\ldots, f_T$:
\[\E\left[\sum_{t=1}^T \ell_t(f_t) - \min_{f^*\in \cclass} \sum_{t=1}^T \ell_t(f^*) \right]= O(B\sqrt{T})\]
\item The run-time of the algorithm is $\tilde{O}\left(\rho B^4 T^2\right)$.\footnote{Ignoring logarithmic factors in $B$ and $T$.}
\item For each $t=1\ldots T$ and a new test example $\xx$, we can with probability $\geq 1-\delta$ estimate $\left<f_t,\fwx{\xx}\right>$ within accuracy $\epsilon_0$ by running \algref{alg:spsub} with parameters $\alpha^{(t)}$, $\{\xx_i\}_{i=1}^t$, $,\xx$ and $m=O(\frac{B^4T}{\epsilon_0^2}\log 1/\delta )$.
The resulting running time for a test point is then $O(\rho m)$.
\end{enumerate}
\end{theorem}
We next turn to the statistical setting, where we provide bounds on the expected performance. Following standard online to batch conversion and \thmref{thm:mainmain} we can obtain the following Corollary \citep[e.g., see][]{shalev2011online}:
\begin{corollary}[Statistical Setting]\label{cor:mainmain} 
The following holds for any $\epsilon > 0$. Run Algorithm \ref{alg:main} as in \thmref{thm:mainmain}, with $T= O( \frac{B^2}{\epsilon^2})$. Let $S=\{(x_t,y_t)\}_{t=1}^T$, be an IID sample drawn from some unknown distribution $\mathbb{P}$. Let $f_S= \frac{1}{T} \sum f_t$. Then the expected loss satisfies:
\[ \E_{S\sim \mathbb{P}} \left[L(f_S)\right] < \inf_{f^* \in \cclass} L(f^*)+\epsilon.\]
The runtime of the algorithm, as well as estimation time on a test example are as defined in \thmref{thm:mainmain}.
\end{corollary}
Proofs of the results are provided in \secref{sec:analysis} and the appendix.

\section{Related Work}
Learning with random features can be traced to the early days of learning \cite{minsky1988perceptrons}, and infinite networks have also been introduced more than 20 years ago \cite{williams1997computing, hornik1993some}. More recent works have considered learning neural nets (also multi-layer) with infinite hidden units using the kernel trick \cite{cho2009kernel, deng2012use, hazan2015steps, heinemann2016improper}. These works take a similar approach to ours but focus on computing the kernel for certain feature classes in order to invoke the kernel trick. Our work in contrast avoids using the kernel trick and applies to any feature class that can be randomly generated. All the above works are part of a broader effort of trying to circumvent hardness in deep learning by mimicking deep nets through kernels \cite{mairal2014convolutional, bouvrie2009invariance, bo2011object, bo2010kernel}, and developing general duality between neural networks and kernels  \cite{daniely2016toward}.

From a different perspective the relation between random features and kernels has been noted by \citet{rahimi2007random} who show how to represent translation invariant kernels in terms of random features. This idea has been further studied \cite{bach2015equivalence, kar2012random} for other kernels as well. The focus of these works is mainly to allow scaling down of the feature space and representation of the final output classifier.

\citet{dai2014scalable} focus on tractability of large scale kernel methods, and their proposed {\em doubly stochastic} algorithm can also be used for learning with random features as we have here. In \citet{dai2014scalable} the objective considered is of the regularized form:$\frac{\gamma}{2} \|f\|^2 + R(f)$, with a corresponding sample complexity of $O(1/(\gamma^2\epsilon^2))$ samples needed to achieve $\epsilon$ approximation with respect to the risk of the optimum of the regularized objective.

To relate the above results to ours, we begin by emphasizing that the bound in \cite{dai2014scalable}  holds for fixed $\gamma$, and refers to optimization of the regularized objective\ignore{\footnote{See there that the generalization bound (with expectation i.e. thm 4, but similar to thm. 6 you need to take sqroot) is of order $O(\frac{Q_1}{\sqrt{t}})$, and $Q_1\in O(Q_0)$ and $Q_0\in O(\theta^2)$ and $\theta \in O(1/\gamma)$ to conclude we obtain a generalization bound of order $O(\frac{1}{t\gamma^4})$.\roi{Please do not erase this foot note even if put in remark}}}. Our objective is to minimize the risk $R(f)$ which is the expected squared loss, for which we need to choose $\gamma = O(\frac{\epsilon}{B^2})$ in order to attain accuracy $\epsilon$ \cite{sridharan2009fast}. Plugging this $\gamma$ into the generalization bound in  \citet{dai2014scalable} we obtain that the algorithm in \citet{dai2014scalable} needs $O(\frac{B^4}{\epsilon^4})$ samples to compete with the optimal target function in the $B$-ball. Our algorithm needs $O(\frac{B^2}{\epsilon^2})$ examples which is considerably better. We note that their method does extend to a larger class of losses, whereas our is restricted to the quadratic loss.

In \citet{rahimi2009weighted}, the authors consider embedding the domain into the feature space $\xx \to \left[\action{\ww_1}{\xx},\ldots,\action{\ww_m}{\xx}\right]$, where $\ww_i$ are IID random variables sampled according to some prior $\p(\ww)$. They show that with $O(\frac{B^2 \log 1/\delta}{\epsilon^2})$ random features estimated on $O(\frac{B^2 \log 1/\delta}{\epsilon^2})$ samples they can compete with the class:
\be
{\cclass}_{\max}= \left\{\xx \to \int \action{\ww}{\xx} f(\ww) d\p(\ww) ~: ~|f(\ww)| \le B\right\} \nonumber
\ee

Our algorithm  relates to the mean square error cost function which does not meet the condition in  \citet{rahimi2009weighted}, and is hence formally incomparable. Yet we can invoke our algorithm to compete against a larger class of target functions. Our main result shows that \algref{alg:main}, using $\tilde{O}(\frac{B^8}{\epsilon^4})$ estimated features and using $O(\frac{B^2}{\epsilon^2})$ samples will, in expectation, output a predictor that is $\epsilon$ close to the best in $\cclass$.
Note that $|f(\ww)| <B$ implies  $\mathbb{E}_{\ww\sim \p} (f^2(\ww)) <B^2$. Hence ${\cclass}_{\max} \subseteq \cclass$. Note however, that the number of estimated features (as a function of $B$) is worse in our case.

Our approach to the problem is to consider learning with a noisy estimate of the kernel. A related setting was studied in \citet{cesa2011online}, where the authors considered learning with kernels when the data is corrupted. Noise in the data and noise in the scalar product estimation are not equivalent when there is non-linearity in the kernel space embedding. There is also extensive research on linear regression with actively chosen attributes \cite{cesa2011efficient,hazan2012linear}. The convergence rates and complexity of the algorithms are dimension dependent. It would be interesting to see if their method can be extended from finite set of attributes to a continuum set of attributes.

\section{Algorithm}\label{sec:alg}
We next turn to present \algref{alg:main}, from which our main result is derived. The algorithm is similar in spirit to Online Gradient Descent (OGD) \cite{Zinkevich03}, but with some important
modifications that are necessary for our analysis.

We first introduce the problem in the terminology of online convex optimization, as in \citet{Zinkevich03}. At iteration $t$ our algorithm outputs a hypothesis $f_t$. It then receives as feedback
$(\xx_t,y_t)$, and suffers a loss $\ell_t(f_t)$ as in \eqref{eq:loss}. The objective of the algorithm is to minimize the regret against a benchmark of $B$-bounded functions, as in \eqref{eq:regret}.

A classic approach to the problem is to exploit the OGD algorithm. Its simplest version would be to update $f_{t+1} \to f_t - \eta \nabla_t$ where
$\eta$ is a step size, and $\nabla_t$ is the gradient of the loss w.r.t. $f$ at $f_t$.  In our case, $\nabla_t$ is given by:
\be
\nabla_t = \left(\scalar{f_t}{\xxtp} - y_t \right) \xxtp
\label{eq:exact_gradient}
\ee
Applying this update would also result in a function $f_t = \sum_{i=1}^t \alpha_i \xxtp$ as we have in \algref{alg:main} (but with different $\alpha_i$ from ours). 
However, in our setting this update is not applicable since the scalar product $\scalar{f_t}{\xxtp}$ is not available. One alternative is to use a stochastic unbiased estimate of the gradient that we denote by $\bar{\nabla}_t$. This induces an update step $\vv_{t+1} \to \vv_t -\eta \bar{\nabla}_t$. One can show that OGD with such an estimated gradient enjoys the following upper bound on the regret $\expect{\sum \ell_t(\vv_t) - \ell_t(\vv^*)}$ for every $\|\vv^*\|\le B$ \citep[e.g., see][]{shalev2011online}:
\begin{equation}\label{eq:ogd}\frac{B^2}{\eta} + \eta\sum_{i=1}^T\expect{\|\nabla_t\|^2}+ \eta\sum_{i=1}^T\Var{\bar{\nabla}_t} ~,\end{equation} 
where $\Var{\bar{\nabla}_t}  = \expect{\|\bar{\nabla}_t-\nabla_t\|^2}$.
We can bound the first two terms using standard techniques applicable for the squared loss \citep[e.g., see][]{zhang2004solving,srebro2010smoothness}. The third term depends on our choice of gradient estimate. There are various choices for such an estimate, and we use a version which facilitates our analysis, as explained below.

Assume that at iteration $t$, our function $f_t$ is given by $f_t = \sum_{i=1}^t \alphat_i \xxtp$. We now want to use sampling to obtain an unbiased estimate of $\scalar{f_t}{\xxtp}$.
 This will be done via a two step sampling procedure, as described in Algorithm \ref{alg:spsub}. First, sample an index $i\in[1,\ldots,t]$ by sampling according to the distribution $q(i) \propto |\alphat_i|$.  Next, for the chosen $i$, sample $\bar{\ww}$ according to $\mu$, and use $\psi(\xx;\bar{\ww}) \psi(\xx_i;\bar{\ww})$ to construct an estimate of $\scalar{\fwx{\xx_i}}{\xxtp}$. The resulting unbiased estimate of $\scalar{\fwx{\xx_i}}{\xxtp}$ is denoted by $E_t$ and given by:
\be
E_t = \frac{\|\alphat\|_1}{m} \sum_{i=1}^m \textrm{sgn}(\alphat_i)\es{\xx_i}{\xx_t}
\ee
The corresponding unbiased gradient estimate is:
\be
\bar{\nabla}_t= \left(E_t -y_t\right) \xx_t
\ee
 
The variance of $\bar{\nabla}$ affects the convergence rate and depends on both $\|\alpha\|_1$ and the number of estimations $m$. We wish to maintain $m=O(T)$ estimations per round, while achieving $O(\sqrt{T})$ regret. 

To effectively regularize $\|\alpha\|_1$, we modify the OGD algorithm so that whenever $E_t$ is larger then $16 B$, we do not perform the usual update. Instead, we perform a shrinking step  that divides $\alphat$ (and hence $f_t$) by $4$. Treating $B$ as constant, this guarantees that $\|\alpha\|_1= O(\eta T )$, and in turn $\textrm{Var}(\bar{\nabla}_t) = O(\frac{\eta^2 T^2}{m})$. Setting $\eta = O(1/\sqrt{T})$, we have that $m=O(T)$ estimations are sufficient.

 The rationale for the shrinkage is that whenever $E_t$ is large, it indicates that $f_t$ is ``far away'' from the $B$-ball, and a shrinkage step, similar to projection, brings $f_t$ closer to the optimal element in the $B$-ball. However, due to stochasticity, the shrinkage step does add a further term to the regret bound that we would need to take care of.
\ignore{
\begin{definition}
A set $K$ in the unit ball of a linear space is $D$-bounded if for every $\xx_1,\xx_2 \in K$ we have that $|\es{\xx_1}{\xx_2}|<D$ a.s.
\end{definition}

\begin{theorem}\label{thm:main} Let $K$ be a $D$-bounded set.
Run Algorithm \ref{alg:main} with parameters $T$, $B\ge 1$, $\eta :=\frac{B}{2\sqrt{T}}$, and $m=((16B+1) D B)^2T\log \gamma$, where  $\gamma = \frac{((16B+1)\eta T+B)^2)}{\eta^2}$. We assume that $\eta<1/8$.

Assume that $\xx_t \in K$ and $y_t\in [-1,1]$ for all $t$ and for each $t$ let $f_t= \sum \alpha^{(t)}_i \fwx{\xx_i}$. Then:
\[\sum_{t=1}^T \ell_t(f_t) - \min_{\|f^*\|\le B}\sum_{t=1}^T \ell_t(f^*)   = O\left(B\sqrt{T}\right).\]
The number of times the algorithm performs estimation of the scalar product is \[Tm= \tilde{O}((DB^2T)^2).\]
Finally, for any $\xx$ the scalar product $\scalar{f_t}{\fwx{\xx}}$ (i.e., the regression function) can be estimated within accuracy $\epsilon_0$ using $O(\frac{D^2B^4 T}{\epsilon_0^2}\log 1/\delta )$ estimations.
\end{theorem}}

\subsection{Analysis}\label{sec:analysis}
In what follows we analyze the regret for \algref{alg:main}, and provide a high level proof of Theorem \ref{thm:mainmain}. The appendix provides the necessary lemmas and a more detailed proof. 
We begin by modifying the regret bound for OGD in \eqref{eq:ogd} to accommodate for steps that differ from the standard gradient update, such as shrinkage. We use the following notation for the regret at iteration $t$:
\be
R_t(\vv^*) = \expect{\sum_{t=1}^T \ell_t(\vv_t)-\ell_t(\vv^*)}
\ee

\begin{lemma}\label{lem:core}
Let $\ell_1,\ldots,\ell_T$ be an arbitrary sequence of convex loss functions, and let $\vv_1,\ldots,\vv_T$ be random vectors, produced by an online algorithm. Assume $\|\vv_i\|\le B_T$ for all $i\le T$. For each $t$ let $\bar{\nabla}_t$ be an unbiased estimator of $\nabla \ell_t(\vv_t)$. Denote $\hat{\vv_t}=\vv_{t-1} - \eta \bar{\nabla}_{t-1}$ and let
\begin{equation}\label{eq:ptvv} P_t(\vv^*) = \prob{\|\vv_t-\vv^*\|> \|\hat{\vv}_t-\vv^*\|}.\end{equation}
For every $\|\vv^*\|\le B$ it holds that :
 \bea\label{eq:main}
R_t(\vv^*) &\leq& \frac{B^2}{\eta} +\eta\sum_{t=1}^T\expect{\|\nabla_t\|^2}+\eta\sum_{t=1}^T\Var{\bar{\nabla}_t} + \nonumber \\
 && \sum_{t=1}^T \frac{(B_T+B)^2}{\eta}\expect{P_{t}(\vv^*)} 
 \eea
\end{lemma}
See \apref{sec:proof_lem_core} for proof of the lemma.  As discussed earlier, the first three terms on the RHS are the standard bound for OGD from \eqref{eq:ogd}. Note that in the standard OGD it holds that $\vv_t=\hat{\vv}_t$, and therefore $P_t(\vv^*)=0$ and the last term disappears. 

The third term will be bounded by controlling $\|\alpha\|_1$. The last term $P_{t}(\vv^*)$ is a penalty that results from updates that stir $\vv_t$ away from the standard update step $\hat{\vv}_t$. This will indeed happen for the shrinkage step. The next lemma bounds this term. See \apref{sec:proof_ptbound} for proof. 
\begin{lemma}\label{thm:ptbound}
Run \algref{alg:main} with parameters $T$, $B\ge 1$ and $\eta<1/8$. Let $\bar{\nabla}_t$ be the unbiased estimator of $\nabla \ell_t(\vv_t)$ of the form
$\bar{\nabla}_t = (E_t -y_t) \xxtp$. Denote $\hat{\vv}_t= \vv_t-\eta \bar{\nabla}_t$ and define $P_t(\vv^*)$ as in \eqref{eq:ptvv}. Then:
\[P_t(\vv^*)\le 2\exp\left(-\frac{m}{(3\eta t)^2}\right)\]
\end{lemma}
The following lemma (see \apref{sec:proof_bound_var} for proof) bounds the second and third terms of \eqref{eq:main}.
\begin{lemma} \label{lemma:bound_var} Consider the setting as in \lemref{thm:ptbound}. Then
$\Var{\bar{\nabla}_t} \le \frac{((16 B+1)\eta t)^2}{m}$ and $\expect{\|{\nabla}_{t}\|^2} \le 2\expect{\ell_t(\vv_t)}$.
\end{lemma}

\paragraph{Proof of \thmref{thm:mainmain}}
Combining Lemmas \ref{lem:core}, \ref{thm:ptbound} and \ref{lemma:bound_var} and rearranging we get:
 \bea\label{eq:almost_main}
&& (1-2\eta)\expect{R_t(\vv^*)} \le \frac{B^2}{\eta}  + 2\eta \sum_{t=1}^T \ell_t(\vv^*) + \\
&& \eta \frac{ ((16B+1)\eta T)^2 T}{m}+ \frac{(B_T+B)^2}{\eta} \sum_{t=1}^TP_t(\vv^*)  \nonumber
\eea
To bound the second term in \eqref{eq:almost_main} we note that:
 \begin{equation}\label{eq:ltbound}\min_{\|\vv^*\|<B} \sum_{t=1}^T \ell_t(\vv^*)\le \sum_{t=1}^T \ell_t(0)\le T.\end{equation}
 We next set $\eta$ and $m$ as in the statement of the theorem. Namely: $\eta =\frac{B}{2\sqrt{T}}$, and $m=((16B+1) B)^2T\log \gamma$, where  $\gamma = \max\left(\frac{((16B+1)\eta T+B)^2)}{\eta^2},e\right)$.
 This choice of $m$ implies that $m>((16B+1)\eta T)^2$, and hence the third term in \eqref{eq:almost_main} is upper bounded by $T$.
 
Next we have that $m> (3\eta t)^2 \log \gamma$ for every $t$, and by the bound on $B_T$ we have that $\gamma> \frac{(B+B_T)^2}{\eta^2}$. Taken together with \lemref{thm:ptbound} we have that:
\begin{equation}\label{eq:ptogd} \frac{(B_T+B)^2}{\eta} \sum_{t=1}^T P_t(\vv^*)\le \eta T.\end{equation}
The above bounds imply that:
\[(1-2\eta)\expect{R_t(\vv^*)}\le \frac{B^2}{\eta}  + 2\eta T +\eta T+\eta T\]

Finally by choice of $\eta$, and dividing both sides by $(1-2\eta)$ we obtain the desired result.


\section{Experiments \label{sec:exp}}
In this section we provide a toy experiment to compare our Shrinking Gradient algorithm to other random feature based methods. In particular, we consider the following three algorithms:
 {\bf Fixed-Random:} Sample a set of $r$ features $\ww_1,\ldots,\ww_r$ and evaluate these on all the train and test points. Namely, all $\xx$ points will be evaluated on the same features. This is the standard random features approach proposed in \citet{rahimi2007random,rahimi2009weighted}.
{\bf Doubly Stochastic Gradient Descent \cite{dai2014scalable}:} Here each training point $\xx$ {samples} $k$ features $\ww_1,\ldots,\ww_k$. These features will from that point on be used for evaluating dot products with $\xx$. Thus, different $\xx$ points will use different features.  
 {\bf Shrinking Gradient:} This is the approach proposed here in \secref{sec:results}. Namely, each training point $\xx$ samples $m$ features in order to calculate the dot product with the current regression function.   

\ignore{
\begin{itemize}
\item {\bf Fixed-Random:} Sample a set of $r$ features $\ww_1,\ldots,\ww_r$ and evaluate these on all the train and test points. Namely, all $\xx$ points will be evaluated on the same features. This is the standard random features approach proposed in \citet{rahimi2007random,rahimi2009weighted}.
\item {\bf Doubly Stochastic Gradient Descent \cite{dai2014scalable}:} Here each training point $\xx$ {samples} $k$ features $\ww_1,\ldots,\ww_k$. These features will from that point on be used for evaluating dot products with $\xx$. Thus, different $\xx$ points will use different features.  
\item {\bf Shrinking Gradient:} This is the approach proposed here in \secref{sec:results}. Namely, each training point $\xx$ samples $m$ features in order to calculate the dot product with the current regression function.   
\end{itemize}
}
In comparing the algorithms we choose $r,k,m$ so that the same overall number of features is calculated. For all methods we explored different initial step sizes and schedules for changing the step size. 

The key question in comparing the three algorithms is how well they use a given budget of random features. To explore this we perform an experiments to simulate the high dimensional feature case. We consider vectors $\xx\in\reals^D$, where a random feature $w$ corresponds to a uniform choice of coordinate $w$ in $\xx$. We work in the regime where $D$ is {\em large} in the sense that $D>T$, where $T$ is the size of the training data. Thus random  sampling of $T$ features will not reveal all coordinates of $\xx$. The training set is generated as follows. First, a training set $\xx_1,\ldots,\xx_T\in\reals^D$ is sampled from a standard Gaussian. We furthermore clip negative values to zero, in order to make the data sparser and more challenging for feature sampling.
Next a weight vector $\aa\in\reals^D$ is chosen as a random sparse linear combination of the training points. This is done in order for the true function to be in the corresponding RKHS.  Finally, the training set is labeled using $y_i = \aa\cdot\xx_i$.

\ignore{
\begin{itemize}
\item A training set $\xx_1,\ldots,\xx_T\in\reals^D$ is sampled from a standard Gaussian. We furthermore clip negative values to zero, in order to make the data sparser and more challenging for feature sampling.
\item A weight vector $\aa\in\reals^D$ is chosen as a random sparse linear combination of the training points. This is done in order for the true function to be in the corresponding RKHS.  
\item The training set is labeled using $y_i = \aa\cdot\xx_i$.
\end{itemize}
}
During training we do not assume that the algorithms have access to $\xx$. Rather they can uniformly sample coordinates from it, which mimics our setting of random features. For the experiment we take $D=550,600,\ldots,800$ and $T=200$. All algorithms perform one pass over the data, to emulate the online regret setting. The results shown in Figure~\ref{linear_experiments} show that our method indeed achieves a lower loss while working with the same feature budget. 

\begin{figure}[th]
	\begin{center}
		\centerline{\includegraphics[scale=0.4]{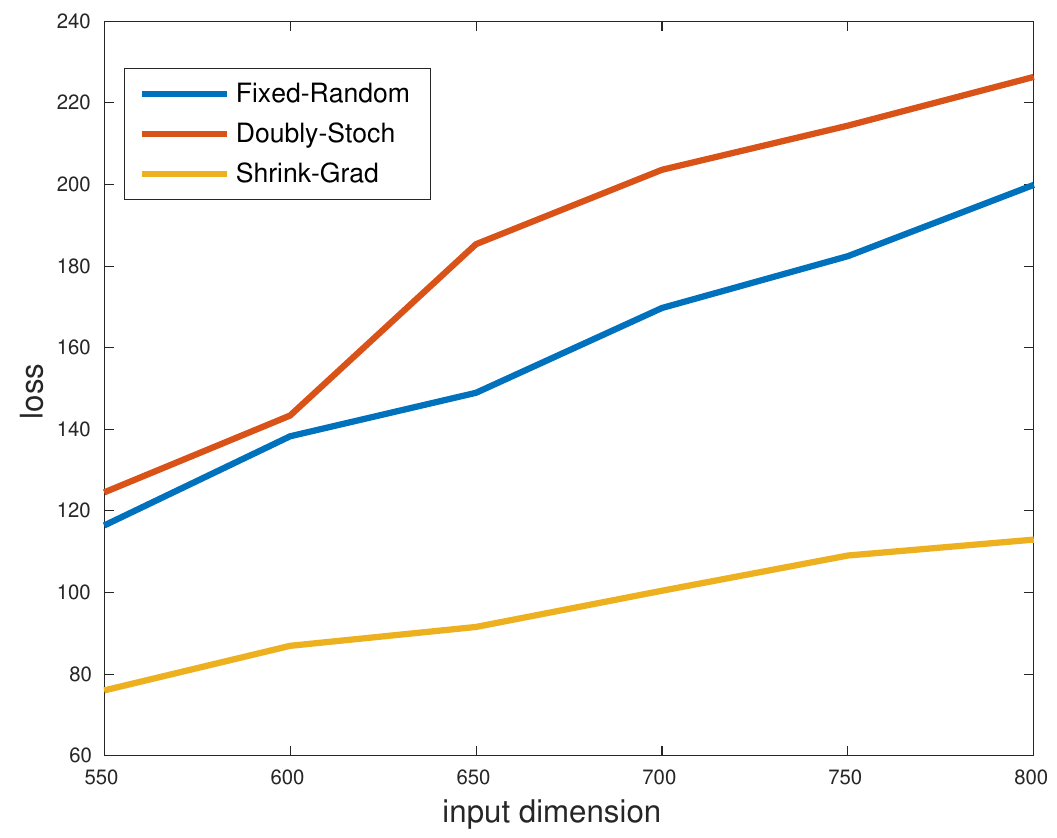}}
		\caption{Comparison of three random feature methods. See \secref{sec:exp} for details.}
		\label{linear_experiments}
	\end{center}
\end{figure} 

\section{Discussion}
We presented a new online algorithm that employs kernels implicitly but avoids the kernel trick assumption. Namely, the algorithm can be invoked even when one has access to only estimations of the scalar product. The problem was motivated by kernels resulting from neural nets, but it can of course be applied to any scalar product of the form we described. As an example of an interesting extension, consider a setting where a learner can observe an unbiased estimate of a coordinate in a kernel matrix, or alternatively the scalar product between any two observations. Our results imply that in this setting the above rates are applicable, and at least for the square loss, having no access to the true values in the kernel matrix is not necessarily prohibitive during training.

The results show that with sample size $T$ we can achieve error of $O(\frac{B}{\sqrt{T}})$. As demonstrated in \citet{shamir2014sample} these rates are optimal, even when the scalar product is computable. To achieve this rate our algorithm needs to perform $\tilde{O}(B^4 T^2)$ scalar product estimations. When the scalar product can be computed, existing kernelized algorithms need to observe a fixed proportion of the kernel matrix, hence they observe order of $\Omega(T^2)$ scalar products. In \citet{cesa2015complexity} it was shown that when the scalar product can be computed exactly, one would need access to at least $\Omega(T)$ entries to the kernel matrix. It is still an open problem whether one has to access $\Omega(T^2)$ entries when the kernel can be computed exactly. However, as we show here, for fixed $B$ even if the kernel can only be estimated $\tilde{O}(T^2)$ estimations are enough. It would be interesting to further investigate and improve the performance of our algorithm in terms of the norm bound $B$.
\ignore{
Another point to consider in future research is the scalability of the predictor at test time. We presented a training algorithm that is comparable with standard kernel methods. Our output predictor can be efficiently estimated with $\epsilon$ accuracy and requires $O(\frac{T}{\epsilon^2})$ generated features. 
Squared loss is often used as a convex surrogate for $0-1$ loss, and for binary classification constant $\epsilon_0$ is enough under appropriate assumptions. It is still interesting to find out if methods such as presented in \cite{rahimi2007random,dai2014scalable} may be used to scale down the representation of the predictor.}

To summarize, we have shown that the {\em kernel trick} is not strictly necessary in terms of sample complexity. Instead, simply sampling random features via our proposed algorithm results in a similar sample complexity. Recent empirical results by \citet{zhang2016understanding} show that using a large number of random features and regression comes close to the performance of the first successful multilayer CNNs \cite{krizhevsky2012imagenet} on CIFAR-10.  Although deep learning architectures still substantially outperform random features, it is conceivable that with the right choice of random features, and scalable learning algorithms like we present here, considerable improvement in performance is possible.   

\ignore{
\section*{Acknowledgement}
The authors would like to thank Tomer Koren for helpful comments and suggestions. Roi Livni  is a recipient of the Google Europe Fellowship in
Learning Theory, and this research is supported in part by
this Google Fellowship, and a Google Research Award.  
}


\appendix
\section{Estimation Concentration Bounds}
In this section we provide concentration bounds for the estimation procedure in \algref{alg:spsub}.
\begin{lemma}\label{lem:spsub}
Run \algref{alg:spsub} with $\alpha$ and, $\{\xx_i\}_{i=1}^T$,  $\xx$, and $m$. Let $\vv=\sum \alpha_i \xxip$. Assume that $|\action{\xx}{\ww}|<1$ for all $\ww$ and $\xx$. Let $E$ be the output of \algref{alg:spsub}. Then $E$ is an unbiased estimator for $\left<\vv,\xxp\right>$ and:
\be
\prob{ |E- \<\vv,\xxp\>| > \epsilon} \le \exp\left(-\frac{m\epsilon^2}{\|\alpha\|_1^2}\right)
\ee
\end{lemma}
\begin{proof}
  Consider the random variables $\|\alpha\|_1 E^{(k)}$  (where $E^{(k)}$ is as defined in \algref{alg:spsub}) and note that they are IID. One can show that $\expect{\|\alpha\|_1 E^{(k)}}=\sum \alpha_i\expect{\action{\xx_i}{\ww}\action{\xx}{\ww}} = \<\vv,\xxp\>$. By the bound on $\action{\xx}{\ww}$ we have that $\left|\|\alpha\|_1E^{(k)}\right| < \|\alpha\|_1$ with probability $1$. Since $E=\frac{1}{m} \sum E^{(k)}$ the result follows directly from Hoeffding's inequality.
\end{proof}
Next,  we bound the $\alphat$ coeffcients and obtain a concentration bound for the estimated dot product $E_t$. 
\begin{lemma}\label{lem:estimation}
The $\alphat$ obtained in \algref{alg:main} satisfies:
\[\|\alpha^{(t)}\|_1 \le   (16B+1)\eta t .\]

As a corollary of this and Lemma \ref{lem:spsub} we have that the function $\vv_t$ satisfies:
\be
\prob{ |E_t- \<\vv_t,\xxtp\>| > \epsilon} \le \exp\left(-\frac{\epsilon^2 m}{((16B+1) \eta t)^2}\right)
\ee

\end{lemma}
\begin{proof}
We prove the statement by induction. We separate into two cases, depending on whether the shrinkage step was performed or
not. 

If $|E_t|\geq 16B$ the algorithm sets $\alpha^{(t+1)}=\frac{1}{4} \alpha^{(t)}$, and: 
\[\|\alpha^{(t+1)}\|_1=\frac{1}{4}\|\alpha^{(t)}\|_1 \le (16 B+1)\eta (t+1)\]

If $|E_t|< 16B$ the gradient update is performed. Since $|y_t|\leq 1$ we have that $|E_t-y_t |<16B +1$ and:
\[ \|\alpha^{(t+1)}\|_1 \le \|\alpha^{(t)}\|_1 +\eta |E_t-y_i |\le (16 B+1)\eta (t+1).\]
\end{proof}

\section{Proofs of Lemmas}

\subsection{Proof of \lemref{lem:core} \label{sec:proof_lem_core}}
First, by convexity we have that
\begin{equation}\label{eq:regret1} 2(\ell_t(\vv_t)-\ell_t(\vv^*)) \le 2\scalar{\nabla_t}{\vv_t -\vv^*}.\end{equation}
Next we upper bound $\scalar{\nabla_t}{\vv_t-\vv^*}$. Denote by $\mathcal{E}$ the event $\|\vv_{t+1}-\vv^*\|>\|\hat{\vv}_{t+1}-\vv^*\|$. Note that:
\bea
&& \expect{\|\vv_{t+1} - \vv^*\|^2}\le \expect{\|\hat{\vv}_{t+1}-\vv^*\|^2} + \nonumber \\
&&  \expect{\|\vv_{t+1}-\vv^*\|^2\big| \mathcal{E} }\cdot  P_{t+1}(\vv^*) \nonumber  \\
&& \le \expect{\|\hat{\vv}_{t+1}-\vv^*\|^2}+(B+B_T)^2 P_{t+1}(\vv^*) \nonumber
\eea

Plugging in $\hat{\vv}_{t+1}= \vv_t-\eta\bar{\nabla}_t$, summing over $t$ and using \eqref{eq:regret1} 
and  $\expect{\|\bar{\nabla_t}\|^2}=\expect{\|\nabla_t\|^2}+\Var{\bar{\nabla_t}}$, we obtain the desired result.
\ignore{
\[\expect{\|\vv_{t+1} - \vv^*\|^2}\le \expect{\|\vv_t-\vv^*\|^2}+\eta^2 \expect{\|\bar{\nabla}_t\|^2} -2\eta\expect{\<\bar{\nabla}_t,\vv_t-\vv^*\>}+(B+B_T)^2P_{t+1}(\vv^*)\]

Dividing by $\eta$ we have that:
\begin{equation}\label{eq:regret2} 2\expect{\<\bar{\nabla}_t,\vv_t-\vv^*\> }\le \frac{\expect{\|\vv_t-\vv^*\|^2} - \expect{\|\vv_{t+1}-\vv^*\|^2}}{\eta} +\frac{1}{\eta} \expect{(B+B_T)^2 P_{t+1}(\vv^*)} +\eta\expect{\|\bar{\nabla}_t\|^2}\end{equation}

Taking  \ref{eq:regret1} and \ref{eq:regret2} and summing we have:

\[ 2 \expect{\sum_{t=1}^T \ell_t(\vv_t)- \ell_t(\vv^*)} \le 
 2 \expect{\nabla_t^\top (f_t-f^*)}
=2\expect{ \bar{\nabla}_t^\top (f_t-f^*)} \le
 \]

\[\frac{\|\vv^*\|^2}{\eta}+ \frac{1}{\eta} \sum_{t=1}^T \expect{(B+B_T)^2 P_{t+1}(\vv^*)} +\eta\sum_{t=1}^T \expect{\|\bar{\nabla}_t\|^2} .\]

Finally note that $\expect{\|\bar{\nabla_t}\|^2}=
\expect{\|\nabla_t\|^2}+\Var{\bar{\nabla_t}}$ to obtain the result.
}
\subsection{Proof for \lemref{thm:ptbound} \label{sec:proof_ptbound}}
To prove the bound in the lemma, we first bound the event $P_t(\vv^*)$ w.r.t to two possible events:
\begin{lemma}\label{lem:ptbound}
Consider the setting as in \lemref{thm:ptbound}. Run \algref{alg:main} and for each $t$ 
consider the following two events:
\begin{itemize}
\item{ $\mathcal{E}^t_1$  \ : \
$|E_t|>16 B$ and $|E_t|>\frac{1}{4\eta}\|\vv_t\|.$
}
\item { $\mathcal{E}^t_2$ \ : \ $|E_t|>16 B$ and $\|\vv_t\|<8 B$.}
\end{itemize}
For every $\|\vv^*\|<B$ we have that $P_t(\vv^*) <\prob{\mathcal{E}_1^t\cup \mathcal{E}_2^t}$.
\end{lemma}
\begin{proof}
Denote the event $|E_t|>16B$ by $\mathcal{E}^t_0$. Note that if $\mathcal{E}^t_0$ does not happen, then $\vv_t=\hat{\vv}_t$. Hence trivially
\[P_t(\vv^*) = \prob{\|\vv_t-\vv^*\| > \|\hat{\vv}_t -\vv^*\| \wedge \mathcal{E}^t_0}\]

We will assume that: (1) $|E_t|>16 B.$, (2) $|E_t|<\frac{1}{4\eta}\|\vv_t\|.$, (3) $\|\vv_t\|>8 B$.
\ignore{
\begin{enumerate}
\item\label{it:1}  $|E_t|>16 B.$
\item\label{it:2} $|E_t|<\frac{1}{4\eta}\|\vv_t\|.$
\item\label{it:3} $\|\vv_t\|>8 B$
\end{enumerate}
}
We then show $\|\vv_{t+1}-\vv^*\|\le\|\hat{\vv}_{t+1}-\vv^*\|.$ 

In other words, we will show that if $\mathcal{E}^t_0$ happens and $\|\vv_{t+1}-\vv^*\|>\|\hat{\vv}_{t+1}-\vv^*\|,$ then either $\mathcal{E}^t_2$ or $\mathcal{E}^t_1$ happened.
This will conclude the proof.

Fix $t$, note that since $|\action{\xx}{\ww}|<1$ we have that $\|\fwx{\xx}\|<1$. We then have:
\bea
\|\hat{\vv}_{t+1}\|&=&\|\vv_t-\eta (E_t-y)\xxtp\| \\
&\ge& \|\vv_t\| -\eta|E_t|-\eta\ge \frac{3}{4}\|\vv_t\|-\eta  \nonumber
\eea
where the last inequality is due to assumption (2) above. We therefore have the following for every $\|\vv^*\|<B$:
\[ \|\hat{\vv}_{t+1} - \vv^*\|\ge \frac{3}{4}\|\vv_t\|-\eta-B\]

On the other hand, if $\vv_{t+1}\ne \hat{\vv}_{t+1}$ then by construction of the algorithm $\vv_{t+1} = \frac{1}{4}\vv_{t}$:
\[\|\vv_{t+1}-\vv^*\| \le \|\vv_{t+1}\|+\|\vv^*\|\le \frac{\|\vv_t\|}{4}+B.\]
Next note that $\eta<2B$ and assumption (3) states $\|\vv_t\|>8 B$. Therefore:
$\frac{1}{2}\|\vv_t\|> 4B>\eta+ 2B$, 
and:
\bean
\|\hat{\vv}_{t+1}-\vv^*\| &\ge& \frac{3}{4}\|\vv_t\|-\eta - B  \\
&=& \frac{1}{4}\|\vv_t\|+ \left(\frac{1}{2}\|\vv_t\| -\eta - 2B\right)+B \\
&\ge& \frac{1}{4} \|\vv_t\| +B\ge\|\vv_{t+1}-\vv^*\|
\eean
\end{proof}

Next we upper bound $\prob{\mathcal{E}_1^t\cup \mathcal{E}_2^t}$. In what follows the superscript $t$ is dropped. 
\newline
{\textbf{A bound for $\prob{\mathcal{E}_1\cap \mathcal{E}_2^c}$:}}
Assume that \[|E_t -\<\vv_t,\xxtp\>| < (\frac{1}{4\eta}-1)8 B.\]
We assume $T$ is sufficiently large and $\eta<\frac{1}{8}$. We have $\frac{1}{4\eta}-1 >1$. Since we assume $\mathcal{E}_2$ did not happen we must have $\|\vv_t\|>8 B$ and
$|E_t -\<\vv_t,\xxtp\>| < (\frac{1}{4\eta}-1)\|\vv\|$, and therefore:
\[E_t-\|\vv\| <|E_t -\<\vv_t,\xxtp\>|< (\frac{1}{4\eta}-1)\|\vv\|.\]
Which implies $E_t< \frac{1}{4\eta}\|\vv\|$, and we get that $\mathcal{E}_1$ did not happen. We conclude that if $\mathcal{E}_1$ and not $\mathcal{E}_2$ then:
\[|E_t -\<\vv_t,\xxtp\>| \ge (\frac{1}{4\eta}-1)8 B.\] Since $\frac{1}{4\eta}-1 > 1$ we have that:
$|E_t -\<\vv_t,\xxtp\>|\ge 8 B$, leading to:
\begin{equation}\label{eq:e1ne2} \prob{\mathcal{E}_1\cap \mathcal{E}_2^c}\le \prob{|E_t -\<\vv_t,\xxtp\>|\ge 8 B}.\end{equation}
\newline
{\textbf{A bound for $\prob{\mathcal{E}_2}$:}}
If $|E_t|>16B$ and $\|\vv_t\|<8B$ then by normalization of $\xxtp$ we have that $\left<\vv_t,\xxtp\right><8B$ and trivially we have that
$|E_t -\<\vv_t,\xxtp\>|\ge 8 B$, and therefore:
\begin{equation}\label{eq:e2} \prob{\mathcal{E}_2}\le \prob{|E_t -\<\vv_t,\xxtp\>|\ge 8 B}.\end{equation}
 Taking \eqref{eq:e1ne2} and \eqref{eq:e2} we have that
 \begin{equation} \prob{\mathcal{E}_2\cup \mathcal{E}_1}\le 2\prob{|E_t -\<\vv_t,\xxtp\>|\ge 8 B}.\end{equation}

   By \lemref{lem:estimation} we have that:
\[P\left(|E_t -\<\vv_t,\xxtp\>|\right) > 8 B)<\]\[
\exp(-\frac{m (8 B)^2}{((16B+1)\eta t)^2})<\exp\left(-\frac{ m}{ (3\eta t)^2}\right)
\]

Taking the above upper bounds together with \lemref{lem:ptbound} we can prove \lemref{thm:ptbound}.

\subsection{Proof of Lemma \ref{lemma:bound_var} \label{sec:proof_bound_var}}
\ignore{
We begin by deriving \corref{cor:standardTerms} that bounds the first two terms in the regret bound. As discussed, this section follows standard techniques.
We begin with an upper bound on $\mathbb{E}(\|\bar{\nabla}_t\|^2)$.
\begin{lemma} Consider the setting as in \lemref{thm:ptbound}. Then

\[\Var{\bar{\nabla}_t} \le \frac{((16 B+1)\eta t)^2}{m},\]
and,
\[\expect{\|{\nabla}_{t}\|^2} \le 2\expect{\ell_t(\vv_t)}.\]
\end{lemma}
\begin{proof}
}
Begin by noting that since $\|\fwx{\xx}\|<1$, it follows from the definitions of $\nabla,\bar{\nabla}$ that $\Var{\bar{\nabla}_{t}}=\expect{\|\bar{\nabla}_{t}-\nabla_{t}\|^2} $
and therefore 
\[
\Var{\bar{\nabla}_{t}} \le \expect{\left(E_t-\<\vv_t,\xxtp\>\right)^2}=\Var{E_t}
\]
By construction (see \algref{alg:spsub}) we have that:
\[\Var{E_t}= \frac{1}{m}\Var{\|\alpha^{(t)}\|_1^2\action{\xx_i}{\ww}{\action{\xx_t}{\ww}}}\]
where the index $i$ is sampled as in \algref{alg:spsub}, and $\action{\xx_i}{\ww}{\action{\xx_t}{\ww}}$ is bounded by $1$. By \lemref{lem:estimation} we have that 
\[\Var{E_t} \le \frac{((16B+1)\eta t)^2}{m}.\]
This provides the required bound on $\Var{\bar{\nabla}_{t}}$. Additionally, we have that
\[ \|\nabla_t\|^2= (\scalar{\vv_t}{\fwx{\xx_t}}-y_t)^2\|\fwx{\xx_t}\|^2  \le  2\ell_t(\vv_t) \]
and the result follows by taking expectation.

\ignore{
We thus have the following corollary that bounds the first three terms in \eqref{eq:main}:
\begin{corollary}\label{cor:standardTerms} With the notations and setting of \lemref{thm:ptbound} we have:
\bea
\frac{B^2}{\eta} + \eta\sum_{t=1}^T \expect{\|{\nabla}_{t}\|^2} +\eta\sum_{t=1}^T \Var{\bar{\nabla}_{t}} 
\le && 
 2\eta \expect{\sum_{t=1}^T \ell_t(\vv_t)-\ell_t(\vv^*)} +\frac{B^2}{\eta} \\
&& 
+2\eta \sum_{t=1}^T \ell_t(\vv^*)+  \eta \sum_{t=1}^T \frac{((16B+1)\eta T)^2}{m} \nonumber
\eea

\end{corollary}
}
\ignore{
\subsection{Proof of \thmref{thm:mainmain}}
Recall that Lemma \ref{lem:core} and Corollary \ref{cor:standardTerms} assume an upper bound $B_T$ on $\norm{f_t}$. We begin by noting that $B_T$ can be bounded as follows, using \lemref{lem:estimation}:
\be
B_T = \max_{t} \|\vv_t\| \le \max_t \|\alpha^{(t)}\|_1\le (16 B+1 ) \eta T. 
\ee

Plugging \corref{cor:standardTerms} into \eqref{eq:main} we obtain:
 \be\label{eq:almost}
(1-2\eta)\expect{\sum_{t=1}^T \ell_t(\vv_t)-\ell_t(\vv^*)}
\le \frac{B^2}{\eta}  + 2\eta \sum_{t=1}^T \ell_t(\vv^*) +\eta\sum_{t=1}^T \frac{ ((16B+1)\eta T)^2}{m}+ \frac{(B_T+B)^2}{\eta} \sum_{t=1}^TP_t(\vv^*)
\ee
To bound the second term we note that:
 \begin{equation}\label{eq:ltbound}\min_{\|\vv^*\|<B} \sum_{t=1}^T \ell_t(\vv^*)\le \sum_{t=1}^T \ell_t(0)\le T.\end{equation}
 We next set $\eta$ and $m$ as in the statement of the theorem. Namely: $\eta =\frac{B}{2\sqrt{T}}$, and $m=((16B+1) B)^2T\log \gamma$, where  $\gamma = \max\left(\frac{((16B+1)\eta T+B)^2)}{\eta^2},e\right)$.
 
Our choice of $m$ implies that $m>((16B+1)\eta T)^2$, and hence the third term in \eqref{eq:almost} is bounded as follows:
\begin{equation}\label{eq:mbound}\eta \sum_{t=1}^T \frac{((16B+1)\eta T)^2}{m} \le \eta T\end{equation}

Next we have that $m> (3\eta t)^2 \log \gamma$ for every $t$, and by the bound on $B_T$ we have that $\gamma> \frac{(B+B_T)^2}{\eta^2}$. Taken together with \lemref{thm:ptbound} we have that:
\begin{equation}\label{eq:ptogd} \frac{(B_T+B)^2}{\eta} \sum_{t=1}^T P_t(\vv^*)\le \eta T.\end{equation}


\[(1-2\eta)\expect{\sum_{t=1}^T \ell_t(\vv_t)-\ell_t(\vv^*)}\le \frac{B^2}{\eta}  + 2\eta T +\eta T+\eta T\]

Finally by choice of $\eta$, and dividing both sides by $(1-2\eta)$ we obtain the desired result.

It remains to show that we can estimate each $\vv_t$ in the desired complexity (the result for the averaged $\vv$ is the same). Each $\vv_t$ has the form $\vv_t= \sum_{t=1}^T \alpha_i^{(t)}\xx_i$, 
By \lemref{lem:spsub} and \lemref{lem:estimation}, running \algref{alg:spsub}  $m$ iterations will lead to a random variable $E$ such that:
\[\prob{|E-\left<\vv_t,\xx\right>|}\le \exp\left(-\frac{\epsilon^2 m}{((16B+1)B\sqrt{T})^2}\right).\] We obtain that order of $O(\frac{B^4 T}{\epsilon^2}\log 1/\delta)$ estimations are enough.
}

\textbf{Acknowledgements}
The authors would like to thank Tomer Koren for helpful discussions.
Roi Livni was supported by funding from Eric and Wendy Schmidt Fund for Strategic Innovation.  This work was supported by the Blavatnik Computer Science Research Fund, the Intel Collaborative Research Institute for Computational Intelligence (ICRI-CI), and an ISF Centers of Excellence grant.
\small
\bibliography{randomfeat}

\begin{thebibliography}{37}
\providecommand{\natexlab}[1]{#1}
\providecommand{\url}[1]{\texttt{#1}}
\expandafter\ifx\csname urlstyle\endcsname\relax
  \providecommand{\doi}[1]{doi: #1}\else
  \providecommand{\doi}{doi: \begingroup \urlstyle{rm}\Url}\fi

\bibitem[Bach(2014)]{bach2014breaking}
Bach, Francis.
\newblock Breaking the curse of dimensionality with convex neural networks.
\newblock \emph{arXiv preprint arXiv:1412.8690}, 2014.

\bibitem[Bach(2015)]{bach2015equivalence}
Bach, Francis.
\newblock On the equivalence between kernel quadrature rules and random feature
  expansions.
\newblock \emph{arXiv preprint arXiv:1502.06800}, 2015.

\bibitem[Bo et~al.(2010)Bo, Ren, and Fox]{bo2010kernel}
Bo, Liefeng, Ren, Xiaofeng, and Fox, Dieter.
\newblock Kernel descriptors for visual recognition.
\newblock In \emph{Advances in neural information processing systems}, pp.\
  244--252, 2010.

\bibitem[Bo et~al.(2011)Bo, Lai, Ren, and Fox]{bo2011object}
Bo, Liefeng, Lai, Kevin, Ren, Xiaofeng, and Fox, Dieter.
\newblock Object recognition with hierarchical kernel descriptors.
\newblock In \emph{Computer Vision and Pattern Recognition (CVPR), 2011 IEEE
  Conference on}, pp.\  1729--1736. IEEE, 2011.

\bibitem[Bouvrie et~al.(2009)Bouvrie, Rosasco, and
  Poggio]{bouvrie2009invariance}
Bouvrie, Jake, Rosasco, Lorenzo, and Poggio, Tomaso.
\newblock On invariance in hierarchical models.
\newblock In \emph{Advances in Neural Information Processing Systems}, pp.\
  162--170, 2009.

\bibitem[Cesa-Bianchi et~al.(2011{\natexlab{a}})Cesa-Bianchi, Shalev-Shwartz,
  and Shamir]{cesa2011efficient}
Cesa-Bianchi, Nicolo, Shalev-Shwartz, Shai, and Shamir, Ohad.
\newblock Efficient learning with partially observed attributes.
\newblock \emph{The Journal of Machine Learning Research}, 12:\penalty0
  2857--2878, 2011{\natexlab{a}}.

\bibitem[Cesa-Bianchi et~al.(2011{\natexlab{b}})Cesa-Bianchi, Shalev-Shwartz,
  and Shamir]{cesa2011online}
Cesa-Bianchi, Nicolo, Shalev-Shwartz, Shai, and Shamir, Ohad.
\newblock Online learning of noisy data.
\newblock \emph{Information Theory, IEEE Transactions on}, 57\penalty0
  (12):\penalty0 7907--7931, 2011{\natexlab{b}}.

\bibitem[Cesa-Bianchi et~al.(2015)Cesa-Bianchi, Mansour, and
  Shamir]{cesa2015complexity}
Cesa-Bianchi, Nicol{\`o}, Mansour, Yishay, and Shamir, Ohad.
\newblock On the complexity of learning with kernels.
\newblock In \emph{Proceedings of The 28th Conference on Learning Theory}, pp.\
   297--325, 2015.

\bibitem[Cho \& Saul(2009)Cho and Saul]{cho2009kernel}
Cho, Youngmin and Saul, Lawrence~K.
\newblock Kernel methods for deep learning.
\newblock In \emph{Advances in neural information processing systems}, pp.\
  342--350, 2009.

\bibitem[Cortes \& Vapnik(1995)Cortes and Vapnik]{cortes1995support}
Cortes, Corinna and Vapnik, Vladimir.
\newblock Support-vector networks.
\newblock \emph{Machine learning}, 20\penalty0 (3):\penalty0 273--297, 1995.

\bibitem[Dai et~al.(2014)Dai, Xie, He, Liang, Raj, Balcan, and
  Song]{dai2014scalable}
Dai, Bo, Xie, Bo, He, Niao, Liang, Yingyu, Raj, Anant, Balcan, Maria-Florina~F,
  and Song, Le.
\newblock Scalable kernel methods via doubly stochastic gradients.
\newblock In \emph{Advances in Neural Information Processing Systems}, pp.\
  3041--3049, 2014.

\bibitem[Daniely et~al.(2014)Daniely, Linial, and
  Shalev-Shwartz]{daniely2014average}
Daniely, Amit, Linial, Nati, and Shalev-Shwartz, Shai.
\newblock From average case complexity to improper learning complexity.
\newblock In \emph{Proceedings of the 46th Annual ACM Symposium on Theory of
  Computing}, pp.\  441--448. ACM, 2014.

\bibitem[Daniely et~al.(2016)Daniely, Frostig, and Singer]{daniely2016toward}
Daniely, Amit, Frostig, Roy, and Singer, Yoram.
\newblock Toward deeper understanding of neural networks: The power of
  initialization and a dual view on expressivity.
\newblock In Lee, D.~D., Sugiyama, M., Luxburg, U.~V., Guyon, I., and Garnett,
  R. (eds.), \emph{Advances in Neural Information Processing Systems 29}, pp.\
  2253--2261. Curran Associates, Inc., 2016.

\bibitem[Deng et~al.(2012)Deng, Tur, He, and Hakkani-Tur]{deng2012use}
Deng, Li, Tur, Gokhan, He, Xiaodong, and Hakkani-Tur, Dilek.
\newblock Use of kernel deep convex networks and end-to-end learning for spoken
  language understanding.
\newblock In \emph{Spoken Language Technology Workshop (SLT), 2012 IEEE}, pp.\
  210--215. IEEE, 2012.

\bibitem[Hazan \& Koren(2012)Hazan and Koren]{hazan2012linear}
Hazan, Elad and Koren, Tomer.
\newblock Linear regression with limited observation.
\newblock In \emph{Proceedings of the 29th International Conference on Machine
  Learning (ICML-12)}, pp.\  807--814, 2012.

\bibitem[Hazan \& Jaakkola(2015)Hazan and Jaakkola]{hazan2015steps}
Hazan, Tamir and Jaakkola, Tommi.
\newblock Steps toward deep kernel methods from infinite neural networks.
\newblock \emph{arXiv preprint arXiv:1508.05133}, 2015.

\bibitem[Heinemann et~al.(2016)Heinemann, Livni, Eban, Elidan, and
  Globerson]{heinemann2016improper}
Heinemann, Uri, Livni, Roi, Eban, Elad, Elidan, Gal, and Globerson, Amir.
\newblock Improper deep kernels.
\newblock In \emph{Proceedings of the 19th International Conference on
  Artificial Intelligence and Statistics}, pp.\  1159--1167, 2016.

\bibitem[Hornik(1993)]{hornik1993some}
Hornik, Kurt.
\newblock Some new results on neural network approximation.
\newblock \emph{Neural Networks}, 6\penalty0 (8):\penalty0 1069--1072, 1993.

\bibitem[Kar \& Karnick(2012)Kar and Karnick]{kar2012random}
Kar, Purushottam and Karnick, Harish.
\newblock Random feature maps for dot product kernels.
\newblock In \emph{International Conference on Artificial Intelligence and
  Statistics}, pp.\  583--591, 2012.

\bibitem[Kivinen et~al.(2004)Kivinen, Smola, and Williamson]{kivinen2004online}
Kivinen, Jyrki, Smola, Alexander~J, and Williamson, Robert~C.
\newblock Online learning with kernels.
\newblock \emph{IEEE transactions on signal processing}, 52\penalty0
  (8):\penalty0 2165--2176, 2004.

\bibitem[Klivans \& Sherstov(2006)Klivans and
  Sherstov]{klivans2006cryptographic}
Klivans, Adam~R and Sherstov, Alexander~A.
\newblock Cryptographic hardness for learning intersections of halfspaces.
\newblock In \emph{Foundations of Computer Science, 2006. FOCS'06. 47th Annual
  IEEE Symposium on}, pp.\  553--562. IEEE, 2006.

\bibitem[Krizhevsky et~al.(2012)Krizhevsky, Sutskever, and
  Hinton]{krizhevsky2012imagenet}
Krizhevsky, Alex, Sutskever, Ilya, and Hinton, Geoffrey~E.
\newblock Imagenet classification with deep convolutional neural networks.
\newblock In \emph{Advances in neural information processing systems}, pp.\
  1097--1105, 2012.

\bibitem[Livni et~al.(2014)Livni, Shalev-Shwartz, and
  Shamir]{livni2014computational}
Livni, Roi, Shalev-Shwartz, Shai, and Shamir, Ohad.
\newblock On the computational efficiency of training neural networks.
\newblock In \emph{Advances in Neural Information Processing Systems}, pp.\
  855--863, 2014.

\bibitem[Mairal et~al.(2014)Mairal, Koniusz, Harchaoui, and
  Schmid]{mairal2014convolutional}
Mairal, Julien, Koniusz, Piotr, Harchaoui, Zaid, and Schmid, Cordelia.
\newblock Convolutional kernel networks.
\newblock In \emph{Advances in Neural Information Processing Systems}, pp.\
  2627--2635, 2014.

\bibitem[Minsky \& Papert(1988)Minsky and Papert]{minsky1988perceptrons}
Minsky, Marvin and Papert, Seymour.
\newblock Perceptrons: an introduction to computational geometry (expanded
  edition), 1988.

\bibitem[Rahimi \& Recht(2007)Rahimi and Recht]{rahimi2007random}
Rahimi, Ali and Recht, Benjamin.
\newblock Random features for large-scale kernel machines.
\newblock In \emph{Advances in neural information processing systems}, pp.\
  1177--1184, 2007.

\bibitem[Rahimi \& Recht(2009)Rahimi and Recht]{rahimi2009weighted}
Rahimi, Ali and Recht, Benjamin.
\newblock Weighted sums of random kitchen sinks: Replacing minimization with
  randomization in learning.
\newblock In \emph{Advances in neural information processing systems}, pp.\
  1313--1320, 2009.

\bibitem[Shalev-Shwartz(2011)]{shalev2011online}
Shalev-Shwartz, Shai.
\newblock Online learning and online convex optimization.
\newblock \emph{Foundations and Trends in Machine Learning}, 4\penalty0
  (2):\penalty0 107--194, 2011.

\bibitem[Shamir(2015)]{shamir2014sample}
Shamir, Ohad.
\newblock The sample complexity of learning linear predictors with the squared
  loss.
\newblock \emph{Journal of Machine Learning Research}, 16(Dec):\penalty0
  3475--3486, 2015.

\bibitem[Srebro et~al.(2010)Srebro, Sridharan, and
  Tewari]{srebro2010smoothness}
Srebro, Nathan, Sridharan, Karthik, and Tewari, Ambuj.
\newblock Smoothness, low noise and fast rates.
\newblock In \emph{Advances in neural information processing systems}, pp.\
  2199--2207, 2010.

\bibitem[Sridharan et~al.(2009)Sridharan, Shalev-Shwartz, and
  Srebro]{sridharan2009fast}
Sridharan, Karthik, Shalev-Shwartz, Shai, and Srebro, Nathan.
\newblock Fast rates for regularized objectives.
\newblock In \emph{Advances in Neural Information Processing Systems}, pp.\
  1545--1552, 2009.

\bibitem[Suykens \& Vandewalle(1999)Suykens and Vandewalle]{suykens1999least}
Suykens, Johan~AK and Vandewalle, Joos.
\newblock Least squares support vector machine classifiers.
\newblock \emph{Neural processing letters}, 9\penalty0 (3):\penalty0 293--300,
  1999.

\bibitem[Suykens et~al.(2002)Suykens, Van~Gestel, and
  De~Brabanter]{suykens2002weighted}
Suykens, Johan~AK, Van~Gestel, Tony, and De~Brabanter, Jos.
\newblock \emph{Least squares support vector machines}.
\newblock World Scientific, 2002.

\bibitem[Williams(1997)]{williams1997computing}
Williams, Christopher.
\newblock Computing with infinite networks.
\newblock \emph{Advances in neural information processing systems}, pp.\
  295--301, 1997.

\bibitem[Zhang et~al.(2016)Zhang, Bengio, Hardt, Recht, and
  Vinyals]{zhang2016understanding}
Zhang, Chiyuan, Bengio, Samy, Hardt, Moritz, Recht, Benjamin, and Vinyals,
  Oriol.
\newblock Understanding deep learning requires rethinking generalization.
\newblock \emph{arXiv preprint arXiv:1611.03530}, 2016.

\bibitem[Zhang(2004)]{zhang2004solving}
Zhang, Tong.
\newblock Solving large scale linear prediction problems using stochastic
  gradient descent algorithms.
\newblock In \emph{Proceedings of the twenty-first international conference on
  Machine learning}, pp.\  116. ACM, 2004.

\bibitem[Zinkevich(2003)]{Zinkevich03}
Zinkevich, Martin.
\newblock Online convex programming and generalized infinitesimal gradient
  ascent.
\newblock In \emph{Machine Learning, Proceedings of the Twentieth International
  Conference}, pp.\  928--936, 2003.

\end{thebibliography}
\bibliographystyle{icml2017}
\normalsize
\end{document}